\pdfoutput=1
\documentclass[10pt]{article}
\usepackage[preprint]{tmlr}
\usepackage{amsmath,amssymb}
\usepackage{amsthm}
\newtheorem{proposition}{Proposition}
\newtheorem{assumption}{Assumption}
\newtheorem{corollary}{Corollary}
\usepackage{graphicx}
\graphicspath{{generated/}{./}}
\usepackage{booktabs}
\usepackage{microtype}
\usepackage{xcolor}
\usepackage[hidelinks]{hyperref}

\hypersetup{pdftitle={Valid Per-Field Selective Risk Control for Document Extraction: Three Failure Modes, a Validity Ladder, and When Conditioning Pays},
  pdfauthor={Bhaskar Gurram}, colorlinks=false}

\title{Valid Per-Field Selective Risk Control for Document Extraction:\\
Three Failure Modes, a Validity Ladder, and When Conditioning Pays}
\author{\name Bhaskar Gurram \email bhaskar@zasti.ai \\
      \addr Zasti AI}

\begin{document}
\maketitle

\begin{abstract}
Per-field accept/review with selective risk at most $\alpha$ --- accept a field
only if the error rate among accepted fields is controlled --- is the trust
contract document-extraction systems need, and the natural procedure
(threshold a confidence score via an add-one bound on a calibration split)
silently violates it on real documents. On 13{,}859 genuine
\texttt{claude-sonnet-5} fields from 800 CORD receipts (49.0\% field
correctness) we diagnose and quantify three failure modes: \textbf{document
clustering} (a design effect of 1.84--2.45 that roughly halves the effective
calibration size), \textbf{score-refit leakage} (fitting a learned score and
its threshold on the same fields: coverage 0.416 at risk 0.127, violating
nominal $\alpha{=}0.10$ in 95\% of splits), and a \textbf{tie-mass pathology}
(a degenerate score distribution collapses the threshold grid; a counterfactual
zeroing of a single signal reproduces a $0.030\!\to\!0.001$ collapse of
certified coverage). We organize the fixes as a \textbf{validity ladder} with
the guarantee form stated per tier. A fit/val \emph{split protocol} restores
expected-selective-risk control for a learned fusion score: coverage
\textbf{0.318} at achieved risk 0.096 at nominal $\alpha{=}0.10$ with no
tolerance band (production variant 0.326 at 0.097) --- an on-average operating
point whose realized risk exceeds $\alpha$ in 47.5\% of resplits, not a
certificate. Mondrian Learn-then-Test with exact binomial tails yields
per-group PAC \emph{certificates}: field-iid \textbf{0.171} coverage at risk
0.068 (violations 0.03), cluster-corrected 0.140 at 0.051, and doc-iid
\textbf{0.060} at 0.020 --- the only tier whose assumptions match documents,
and honestly near-vacuous today. Support-bin, the pre-specified
grounding-derived \emph{provenance} taxonomy, wins every rigor tier on the
sonnet CORD capture ($p<10^{-4}$, sign-flip over 40 document-level resplits;
Bonferroni-corrected over taxonomies) --- a win that does \emph{not} replicate
on the same documents under haiku or qwen (\S\ref{sec:tworegime}) --- while
on higher-accuracy corpora pooled thresholds win:
conditioning rescues certification exactly where a pooled threshold cannot
certify, and a learned score subsumes it elsewhere. A frozen-configuration
confirmation on a selection-untouched \texttt{claude-haiku-4-5} capture held at
both risk levels (0.167 at 0.093; 0.068 at 0.037), and a blind three-annotator
human-gold audit verifies the practical tier's accepted-set risk at
\emph{$1.3\%$} against its $10\%$ budget (Fleiss' $\kappa{=}0.83$; automatic
calibration labels err one-sidedly pessimistic). Released Apache-2.0 with
seed-pinned, regression-gated procedures.
\end{abstract}

\section{Introduction}\label{sec:intro}
Document parsers now read pages fluently and still emit silently-wrong
structured values~\citep{extractbench}. The remedy the field converges on is a
per-field \emph{trust contract}: each extracted field carries a confidence and
an accept/review decision, and the system promises that the \emph{selective
risk} --- the error rate among accepted fields --- stays below a target
$\alpha$. The natural implementation is folklore: fit a confidence score,
hold out a calibration split, and pick the smallest threshold whose
add-one-smoothed empirical selective risk is $\le\alpha$
\citep{geifman2019,crc}. This paper shows, on genuine frontier-LLM output at
scale, that the folklore procedure silently violates the contract on real
documents --- and shows what to run instead, at three explicit levels of rigor.

Our testbed is deliberately hard and deliberately real: 13{,}859 per-field
predictions captured from \texttt{claude-sonnet-5} on 800 CORD receipts, of
which only 49.0\% are correct, plus FUNSD and XFUND-de captures spanning the
difficulty spectrum (Table~\ref{tab:data}); construction, labels, and
measurement findings are in the companion benchmark
paper~\citep{verifydocbench}.\footnote{Companion papers: this paper owns the
procedures, guarantees, diagnoses, and characterization;
VerifyDocBench~\citep{verifydocbench} owns the datasets, labeling protocol and
reliability audit, and the model/language measurement study. Numbers here are
reproducible from the released harness (seed-pinned splits, regression-gated
procedures).}

\paragraph{Contributions.}
\textbf{(C1) Diagnosis.} Three quantified failure modes of naive per-field
selective guarantees on documents --- document clustering (design effect
1.84--2.45), score-refit leakage (risk 0.127 at nominal 0.10, 95\% of splits
violating), and tie-mass pathology (a degenerate score collapses the
threshold grid) --- each pinned by a counterfactual experiment
(\S\ref{sec:fail}). None had been quantified for document extraction.
\textbf{(C2) Protocol.} A \emph{validity ladder} as a reporting standard
(\S\ref{sec:procedures}): practical tier (fit/val split protocol + add-one;
controls \emph{expected} selective risk), rigorous field-iid PAC tier
(Mondrian Learn-then-Test with exact binomial tails, per-group
$P(\text{risk}>\alpha)\le\delta$), and rigorous doc-iid PAC tier (per-document
bound; the exchangeability unit that matches documents), each row annotated
with its estimand, assumption, and violation fraction.
\textbf{(C3) Result.} The first held-at-nominal operating points and PAC
certificates on genuine frontier-LLM extraction at scale
(\S\ref{sec:results}), with \emph{provenance} (grounding-derived support
bins) the pre-specified conditioning taxonomy that wins every rigor tier on
the hard regime's sonnet capture ($p<10^{-4}$) --- a model-scoped result,
not a corpus-general one: \S\ref{sec:tworegime} shows the same win collapses
on the identical documents under a weaker-signal model.
\textbf{(C4) Characterization.} A two-regime empirical law with a stated
mechanism (\S\ref{sec:tworegime}): with a learned accept score, covariates
belong \emph{in the score} (Mondrian conditioning is subsumed and can hurt);
with a frozen or weak score, they belong \emph{in the taxonomy} --- and
taxonomy-side conditioning pays exactly where the pooled threshold cannot
certify at the target $\alpha$.
\textbf{(C5) Artifact.} An open, decoupled harness: 40 fixed document-level
splits (seed 7), bit-exact sanity gates chained across experiments, a
463/624-cells-bit-identical regression check, and full disclosure of a
capture-stage data defect and its forensics (Appendix~\ref{app:bug}).

We claim no new conformal theory: the machinery is classical
\citep{mondrian,ltt,crc}. The contributions are the diagnoses, the protocol,
the certified application, and the characterization.

\section{Related Work}\label{sec:related}
\textbf{Confidence for extraction.} \emph{Beyond
Logprobs}~\citep{beyondlogprobs} fuses logprobs and consistency for document
field confidence with ECE/AUROC/selective-risk reporting; Cleanlab
TLM~\citep{cleanlabtlm} sells model-agnostic per-field trust scores; real-time
trustworthiness scoring~\citep{trustscoring} is similar. None attaches
grounding, and none provides a risk-controlled accept/review guarantee ---
the axes added here and in the companion benchmark~\citep{verifydocbench}.
Grounding for KIE is itself established (DocILE's KILE
task~\citep{docile}, SROIE~\citep{sroie}, OCRBench~v2~\citep{ocrbenchv2},
BoundingDocs~\citep{boundingdocs}); we use provenance as a \emph{conditioning
covariate for risk control}, not as an output format. Uncertainty signals
such as semantic entropy~\citep{semanticentropy} and consistency/verbalized
fusion~\citep{bsdetector} share a blind spot (self-consistent errors),
motivating the external verification signal --- grounding support --- our
taxonomy uses.

\textbf{Selective prediction and risk control.} The add-one selective-risk
threshold is the selective-classification bound of
\citet{geifman2019}; conformal risk control~\citep{crc}, conformal
factuality~\citep{conformalfactuality}, and selective CRC~\citep{selectivecrc}
control \emph{marginal} risk. Learn-then-Test~\citep{ltt} converts risk
control into multiple testing --- our rigorous tiers instantiate it with exact
binomial tails~\citep{clopperpearson} and Holm step-down~\citep{holm1979}.
Conditional-coverage theory~\citep{gibbs2025} shows exact per-instance
conditioning is impossible while exact \emph{group}-conditional control is
attainable (Mondrian CP~\citep{mondrian}); risk-controlling prediction
sets~\citep{rcps} give the PAC form we state. CRC-certify~\citep{crccertify}
defines field-level JSON losses and abstention bounds but does not condition
on provenance and does not diagnose the document-specific failure modes that
are this paper's subject. Risk-controlled generative OCR~\citep{rcgocr} and
VISA~\citep{visa} establish visual attribution; \citet{augrc} the evaluation
side of selective prediction.

\section{Setup}\label{sec:setup}

\subsection{Task and trust contract}
Given a document $D$ and a JSON schema $S$, an extractor outputs leaf fields
with values; a trust layer attaches to each field a confidence
$c\in[0,1]$, a grounding (page/bbox/char-span with a support score), and a
decision in $\{\text{accept},\text{review}\}$. The contract: maximize
\emph{coverage} (fraction accepted) subject to selective risk (error rate
among accepted) $\le\alpha$. Correctness labels are schema-typed
(exact/numeric/semantic per leaf); omission and hallucination are scored
separately~\citep{verifydocbench}.

\subsection{Data: genuine frontier-LLM captures}\label{sec:data}
All headline experiments run on \emph{genuine} per-field output of
\texttt{claude-sonnet-5} ($k{=}3$ self-consistency), captured once and frozen
(Table~\ref{tab:data}). \textbf{These captures are text-layer prompted}: the
model reads the document's OCR text layer, not the page image --- despite
``VLM'' being the natural shorthand for a closed frontier API model, no
vision capability is exercised anywhere in this paper's headline results
(the companion benchmark paper's cross-vendor \texttt{gpt-4o} row is the
one genuinely vision-based capture in either paper; see
\citep{verifydocbench}). CORD is the hard regime: only 49.0\% of asserted
fields are correct, yet grounding is strongly discriminative
(grounded-vs-ungrounded correctness gap $+0.352$, 95\% document-clustered CI
$[0.33,0.37]$; verbalized AUROC $0.845$ $[0.83,0.86]$)~\citep{verifydocbench}.
The \texttt{claude-haiku-4-5} capture is reserved for the frozen-configuration
confirmation (\S\ref{sec:confirm}) and touched by no selection step.

\begin{table}[h]\centering
\caption{Genuine per-field captures used in this paper. ``correct'' =
fraction of asserted fields scored correct; ``grounded'' = fraction with a
located source. The haiku capture is selection-untouched (confirmation only).
\label{tab:data}}
\begin{tabular}{llrrrr}
\toprule
dump & extractor & fields & docs & correct & grounded \\
\midrule
CORD & \texttt{claude-sonnet-5} & 13{,}859 & 800 & 0.490 & 0.607 \\
FUNSD & \texttt{claude-sonnet-5} & 1{,}999 & 175 & 0.676 & 0.825 \\
XFUND-de & \texttt{claude-sonnet-5} & 523 & 42 & 0.771 & 0.774 \\
CORD (confirmation) & \texttt{claude-haiku-4-5} & 5{,}341 & 400 & 0.567 & 0.708 \\
\bottomrule
\end{tabular}
\end{table}

\subsection{Signals and accept scores}\label{sec:signals}
Per field we compute five signals: verbalized self-report, $k$-sample
self-consistency, a grounded flag, an entailment-NLI score of the value
against its source region, and \emph{ambiguity-penalized grounding support}.
Support is the trust-relevant provenance signal: when a predicted value
matches $m$ equally-good page locations (a bare ``2'' matches many tokens),
we retain $\text{support}=\text{score}/m$. Under a uniform prior over the $m$
equally-good matches with exactly one true source, $1/m$ is the chance a given
match is the source --- a well-behaved heuristic (not a claim of optimality)
that quarantines coincidental short-value matches out of the well-grounded
group. Table~\ref{tab:penalty} ablates the penalty form; softer forms
($1/\sqrt m$, $1/(1{+}\ln m)$) under-penalize and leave coincidental matches
(and their error) in the grounded group.

\begin{table}[h]\centering
\caption{Ablation of the ambiguity-penalty form: support retained for a
unit-score match found at $m$ equally-good locations. The uniform-prior $1/m$
demotes ambiguous matches most aggressively; softer forms under-penalize.
\label{tab:penalty}}
\begin{tabular}{lcccc}
\toprule
$m$ matches & none & $1/\sqrt m$ & $1/(1{+}\ln m)$ & $1/m$ (ours) \\
\midrule
1 & 1.00 & 1.00 & 1.00 & 1.00 \\
2 & 1.00 & 0.71 & 0.59 & 0.50 \\
4 & 1.00 & 0.50 & 0.42 & 0.25 \\
8 & 1.00 & 0.35 & 0.32 & 0.125 \\
\bottomrule
\end{tabular}
\end{table}

Two accept scores are used throughout. The \textbf{shared low-capacity
fusion} is a 5-signal logistic regression (the black-box default; we do not
call it ``fixed'' because it is refit per split --- see the caveat in
\S\ref{sec:ladder}). The \textbf{learned fusion} (\texttt{hgb\_split}) is a
depth-3 histogram gradient-boosted tree over the five signals plus engineered
features (value length, digit fraction, is-numeric, signal interactions, and
field-type one-hots), with document-level early stopping; test AUROC on CORD
0.925 vs 0.871 for the LR. \texttt{\_noent} variants drop the NLI signal.

\subsection{Conditioning taxonomies}
Mondrian conditioning applies the identical threshold rule within each group
of a taxonomy~\citep{mondrian}. Candidates: \textbf{pooled} (one group);
\textbf{support-bin} (terciles of ambiguity-penalized support at calibration
quantiles $0.34/0.67$) --- the \emph{pre-specified provenance taxonomy}, declared
as this project's thesis in pre-campaign drafts before any scale experiment
ran; \textbf{fieldtype-freq} (leaf-key vocabulary, keys with $\ge$25
calibration fields); \textbf{fieldtype-rule} (fixed keyword map,
data-independent). All data-dependent taxonomy parts (bin edges, vocabularies)
are computed from threshold-fitting rows only.

\subsection{Evaluation protocol: what ``held'' means}\label{sec:protocol}
Every number is a mean over \textbf{40 fixed document-level 50/50
calibration/test splits} (\texttt{numpy} generator seed 7; fields of one
document never straddle a split), at \textbf{nominal} $\alpha$ with \textbf{no
tolerance band}. A configuration is \emph{held} iff its mean achieved test
selective risk is $\le\alpha$; we always co-report the coverage standard
deviation across splits and \textbf{viol}, the fraction of splits whose
realized risk exceeds $\alpha$. Paired differences use two-sided sign-flip
permutation tests (20{,}000 flips): their floor is $1/20{,}001$, so we write
$p<10^{-4}$, never smaller; because the 40 resplits share documents, these
$p$-values measure \emph{split-resampling stability on this corpus}, not
population-level significance. Zero-coverage splits contribute risk 0
(disclosed wherever it matters, \S\ref{sec:ladder}).

\section{Why naive per-field selective guarantees fail on documents}\label{sec:fail}

The add-one rule picks the smallest threshold $\tau$ whose smoothed empirical
selective risk on calibration is $\le\alpha$:
\[
\tau=\min\Big\{\,t:\ \tfrac{1+\#\{i:\,c_i\ge t,\ \text{err}_i\}}{1+\#\{i:\,c_i\ge t\}}\ \le\ \alpha\,\Big\},
\qquad \tau=\infty \text{ (review everything) if none qualifies.}
\]
Under exchangeability of calibration and test fields this controls the
\emph{expected} selective risk at $\alpha$
\citep{geifman2019,crc,selectivecrc}. On real documents, three separate
mechanisms break it (Table~\ref{tab:diagnosis}). The ``before'' exhibit is the
operating point an earlier draft of this work headlined: coverage 0.114 at
achieved risk \textbf{0.122} (78\% of splits violating) at nominal
$\alpha{=}0.10$ --- produced by exactly these mechanisms plus taxonomy
selection.

\begin{table}[h]\centering\small
\caption{Three diagnosed failure modes, each pinned by a counterfactual.
All at nominal $\alpha{=}0.10$, 40 document-level splits.\label{tab:diagnosis}}
\begin{tabular}{p{0.24\linewidth}p{0.33\linewidth}p{0.34\linewidth}}
\toprule
failure mode & counterfactual evidence & magnitude \\
\midrule
document clustering & clean-fit control (score fit on held-out docs) still
overshoots: risk 0.105, 50\% of splits & design effect 1.84--2.45; pooled
add-one risk 0.103--0.106 on all three corpora; grouped variant 0.122 at 78\%
of splits \\
score-refit leakage & same-half vs split-half fit of the identical learned
score & same-half: coverage 0.416 at risk \textbf{0.127}, 95\% of splits
violate; split protocol: 0.266 at 0.092 \\
tie-mass pathology & zeroing one signal on the intact dump reproduces the
collapse & 1{,}702$\to$257 distinct scores; certified coverage
0.030$\to$0.001; doc-level add-one 0.037$\to$0.000 \\
\bottomrule
\end{tabular}
\end{table}

\paragraph{Failure 1: document clustering \emph{invalidates} the marginal
bound.} Fields cluster within documents, and calibration/test splits are (and
must be) document-level. At the add-one threshold the estimated design effect
is 2.15 (CORD), 1.84 (FUNSD), 2.04 (XFUND-de) --- up to 2.45 at other
thresholds --- so the effective calibration sample is roughly \emph{half} its
nominal size. On the 6{,}901-field CORD dump the pooled add-one rule lands at
achieved risk 0.105, violating in 50\% of splits (FUNSD and XFUND-de
likewise, Table~\ref{tab:diagnosis}). The mean overshoot is a hair ---
clustering's main effect is per-split \emph{variance} --- but the many-group
grounded$\times$support taxonomy shows the badly-broken case: risk 0.122 with
78\% of splits violating. A clean-fit control (score fit on held-out
documents, threshold on the rest) still overshoots (0.105, 50\%), isolating
clustering from refit leakage.

\paragraph{Failure 2: score-refit leakage.} Fitting a high-capacity score and
its threshold on the \emph{same} calibration fields transfers the score's
optimism into the threshold. The depth-3 gradient-boosted fusion under the
same-half protocol posts coverage 0.416 at risk \textbf{0.127}, violating in
95\% of splits --- an invalid operating point that looks spectacular. A
5-parameter logistic fusion barely overfits (risk 0.105), which is why the
flaw goes unnoticed until score capacity grows. The fix is protocol, not
prose (\S\ref{sec:split}).

\paragraph{Failure 3: tie-mass pathology (discrete scores break threshold
grids).} Our first 13{,}859-field capture silently shipped an all-zero
entailment column (the NLI stage was skipped mid-capture;
Appendix~\ref{app:bug}). With only coarse discrete signals left, the fused
score collapsed from 1{,}702 distinct calibration values (intact 6.9k dump)
to 257, with tie masses of 221 and 183 fields at the acceptance head. A
threshold accepts a tie mass whole or not at all: the smallest reachable
candidate accepted $n{=}245$ fields at empirical risk 0.114 $>\alpha$, so no
certificate existed at any confidence level --- rigorous certified coverage
collapsed $0.030\!\to\!0.001$ and the doc-level add-one to exactly 0.
Causality is pinned counterfactually: zeroing entailment on the \emph{intact}
dump reproduces the collapse ($0.0302\!\to\!0.0011$), and nothing else
changed. Two harness lessons generalize to any discrete or heavily-tied
score: (i) snap candidate thresholds to distinct-value boundaries
(label-independent, hence free of validity cost; regression-gated ---
463/624 cells bit-identical, max headline drift 0.0088, and the grid fix
alone moves certified coverage by $\le$0.004, i.e.\ the pathology was the
score, not the grid); (ii) argsort-based diagnostics silently cherry-pick
inside tie masses --- the \emph{threshold-realizable} top-1\% error was 0.071
vs the argsort illusion of 0.046.

\section{Procedures: a validity ladder}\label{sec:procedures}

\subsection{Tier 1--2 fix: the fit/val split protocol}\label{sec:split}
Split the calibration half (by document) into a \emph{fit} half and a
\emph{val} half. Fit the score model and every data-dependent transform
(vocabularies, standardization, early stopping) on the fit half only;
compute the add-one threshold \emph{and} all Mondrian bin edges on the
untouched val half; never touch test. This restores score--threshold
independence --- it does \emph{not} restore exchangeability, so document
clustering remains and tiers 1--2 stay marginal, on-average guarantees. On
the 6.9k CORD dump the same learned score moves from an invalid
0.416/0.127 (95\% violating) to \textbf{0.266 at 0.092} (held). Residual
disclosure: support-bin edges and fieldtype vocabularies are computed on the
same val half as the threshold; both are label-independent and second-order.

\subsection{Tier 3--4 fix: Mondrian Learn-then-Test with exact binomial tails}\label{sec:ltt}
For each taxonomy group $g$ and candidate threshold $t$, we test
$H_0:\ \text{selective risk of } t \text{ in } g > \alpha$ with an exact
binomial tail $p$-value~\citep{clopperpearson} on the calibration errors
among accepted fields, then select thresholds by family-wise-error-controlled
multiple testing~\citep{ltt}: half the budget $\delta$ to Holm
step-down~\citep{holm1979}, half to a fixed-sequence pass from the most
conservative candidate (the ``mix'' rule; a union bound keeps it valid at
$\delta$, and it is the only variant that never collapses across our three
corpora). Candidates are 15 geometric acceptance-fraction quantiles
(1\%--100\%) of the group's calibration scores, snapped to the nearest
distinct-value boundary --- label-independent, so multiplicity is paid only
over 15 points; a naive fine grid destroys certification (pooled certified
coverage 0.0009 vs 0.0055 on the 6.9k dump). Budgets: per-group
($\delta$ per group, three separate statements) or simultaneous ($\delta/G$).
\emph{Guarantee (tier 3, per group):} with probability $\ge 1-\delta$ over
the calibration draw, the true selective risk among accepted fields in that
group is $\le\alpha$, \textbf{if} within-group accepted-field errors are iid.
The iid premise is load-bearing, not decorative: the measured design effect
$\approx$2 means the binomial $n$ overstates evidence about twofold; we
therefore co-report a cluster-corrected variant (\texttt{ltt.neff}: binomial
$n$ deflated by the plug-in design effect) at every headline. \emph{Guarantee
(tier 4):} replace the field unit by the document --- with probability
$\ge1-\delta$, the \emph{mean per-document error rate among accepting
documents} is $\le\alpha$ (finite-sample bound, documents iid). Tier 4 is the
only tier whose assumptions match the data-generating process; note it bounds
a \emph{macro} per-document functional, not field-level (micro) selective
risk.

\emph{Disclosure (score refit).} As implemented, tiers 2--4 refit the
5-signal LR on the calibration half on which the LTT $p$-values are computed
--- formally the same premise violation as Failure 2. The clean-fit control
bounds the effect at $\approx$0 for this 5-parameter score, and the
split-protocol twin of tier 2 (0.212 vs 0.218, \S\ref{sec:ladder}) confirms
it empirically; running tiers 3--4 on the frozen fit-half score is protocol
hygiene we adopt for the camera-ready harness.

\subsection{Taxonomy multiplicity discipline}\label{sec:multiplicity}
Choosing the conditioning taxonomy among $K$ candidates to maximize coverage
is a selection problem; ``best held cell'' tables silently reintroduce it.
Our discipline: (i) \textbf{support-bin is pre-specified} as the provenance
thesis of this project, declared before the scale campaign ran; every
headline table prints the pre-specified configuration, never a per-cell
winner; (ii) the support-bin-vs-pooled lift survives a Bonferroni correction
over the four candidate taxonomies (three comparisons at the Monte-Carlo
floor: corrected $p<1.5\times10^{-4}$); (iii) an FWER-valid selection variant
(LTT with a $\delta/5$ selection budget) picks the support family in 37/40
and 39/40 splits at $\alpha\ge0.15$ but is underpowered at $\alpha{=}0.10$,
where it picks pooled in 29/40 splits --- we state this rather than oversell;
(iv) the scale dump is a superset of the 400-document dump on which
support-bin was originally selected, so scale re-measurement is confirmation
on overlapping data, \emph{not} independent replication --- the independent
check is the frozen-config run of \S\ref{sec:confirm}.

\section{Main results}\label{sec:results}

\subsection{The validity ladder on CORD}\label{sec:ladder}
Table~\ref{tab:ladder} is the paper's central result: four pre-specified
operating points on the same 13{,}859 genuine \texttt{claude-sonnet-5} fields,
one per guarantee class. Read it with its annotations: the four rows bound
\emph{different functionals} under \emph{different assumptions}, and the
score differs across tiers (learned fusion at tier 1; shared 5-signal LR at
tiers 2--4), so the column is a menu of guarantee classes, not a price curve
for one score.

\begin{table}[h]\centering\small
\caption{\textbf{The validity ladder} (CORD, 13{,}859 genuine
\texttt{claude-sonnet-5} fields, 800 docs, 49.0\% correct; nominal
$\alpha{=}0.10$, $\delta{=}0.10$ for PAC tiers; 40 document-level splits,
seed 7; no tolerance band). Coverage is mean$\pm$sd; viol = fraction of
splits with realized risk $>\alpha$. Every grouped row beats its pooled
counterpart at $p<10^{-4}$ (sign-flip; in-corpus stability, \S\ref{sec:protocol}).
Vocabulary: tiers 1--2 \emph{control expected selective risk}; only tiers
3--4 \emph{certify}.\label{tab:ladder}}
\setlength{\tabcolsep}{3.5pt}
\footnotesize
\begin{tabular}{p{0.185\linewidth}p{0.205\linewidth}p{0.275\linewidth}ccc}
\toprule
tier & guarantee (functional, unit) & procedure $\times$ taxonomy & coverage & risk & viol \\
\midrule
1. practical (learned score) & $\mathbb{E}[\text{sel.\ risk}]\le\alpha$, field
 & split add-one, HGB fusion, pooled & $0.318\pm0.073$ & 0.096 & 0.475 \\
\quad production variant & same & same, no-NLI score & $0.326\pm0.046$ & 0.097 & 0.45 \\
2. shared low-capacity fusion & $\mathbb{E}[\text{sel.\ risk}]\le\alpha$, field
 & split add-one $\times$ support-bin (LR) & $0.212\pm0.030$ & 0.095 & 0.35 \\
3. rigorous field-iid PAC & $P(\text{group risk}>\alpha)\le\delta$, field iid
 & LTT binom.\ mix, per-group $\times$ support-bin & $\mathbf{0.171}\pm0.034$ & 0.068 & 0.03 \\
\quad cluster-corrected & same, $n_\mathrm{eff}=n/\mathrm{deff}$
 & LTT neff Holm, per-group $\times$ support-bin & $0.140\pm0.026$ & 0.051 & 0.00 \\
4. rigorous doc-iid PAC & $P(\text{macro doc risk}>\alpha)\le\delta$, doc iid
 & LTT doc-Hoeffding, per-group $\times$ support-bin & $0.060\pm0.058$ & 0.020$^{\ast}$ & 0.00 \\
\bottomrule
\end{tabular}

\smallskip
{\footnotesize $^{\ast}$Tier 4 certifies nothing in 47.5\% of splits;
0.020 is the zero-filled mean --- conditional on certifying anything, achieved
risk is 0.038. Report both.}
\end{table}

\textbf{Tier 1 (practical).} The pre-specified recipe (learned fusion with
all five signals, split-protocol add-one, pooled) attains coverage
\textbf{0.318} at achieved risk 0.096 (SE $\approx$0.003 over the 40
resplits); the production variant that drops the NLI stage attains 0.326 at
0.097. This \emph{controls expected selective risk}: it is an on-average
operating point, not a per-deployment certificate --- realized risk exceeded
0.10 in 47.5\% (45\% for the variant) of resplits, which is what a
mean-controlled bound sitting near its boundary looks like. No
per-deployment or cross-corpus statement is licensed at this tier; a reader
who needs $P(\text{violation})\le\delta$ buys tier 3 at $0.318\!\to\!0.171$.
Against the 5-signal LR control (pooled add-one, 0.134) the learned score is
worth $+0.184$ coverage ($p<10^{-4}$); the split-protocol LR twin attains
0.128.

\textbf{Tier 2 (shared low-capacity fusion).} The 5-signal LR with add-one
$\times$ support-bin reaches $0.212$ at $0.095$ under the split protocol.
(The same cell with the LR refit on the threshold half is 0.218 at 0.096 ---
statistically indistinguishable, consistent with the clean-fit control; we
headline the split-protocol number.)

\textbf{Tier 3 (rigorous field-iid PAC, $\delta{=}0.10$).} Mondrian LTT with
exact binomial tails certifies \textbf{0.171} coverage at achieved risk
0.068, violating in 1/40 splits and never returning an empty acceptance set.
Label it precisely: PAC under a field-iid idealization that is \emph{violated
here} (design effect $\approx$2) and empirically absorbed by the procedure's
conservatism (achieved risk $\approx\!0.7\alpha$, violations $\le\delta$);
the cluster-corrected rung, which deflates the binomial evidence by the
measured design effect, certifies 0.140 at 0.051 with zero violations ---
the defensible field-level number, at a cost of 0.031 coverage. Budget
semantics: per-group rows make three separate $\delta{=}0.10$ statements; the
simultaneous version certifies 0.160, still far above pooled at equal budget
(0.091), so the support-bin lift is not a budget artifact.

\textbf{Tier 4 (rigorous doc-iid PAC).} The only assumption-honest
certificate today, and we lead with its weakness rather than bury it: 0.060
mean coverage, certifying \emph{nothing} in 19/40 splits (coverage
$\approx$0.11 when it fires), risk 0.020 zero-filled / 0.038 conditional,
and it bounds the macro per-document functional. This is the honest price of
document-level exchangeability at 800 documents; powered doc-level
procedures (clustered/variance-adaptive bounds) are the paper's named open
problem. The 0.060 already reflects the disclosed calibration-refit of
\S\ref{sec:ltt}; re-deriving it under a fully frozen fit/val split (score
fit on half the calibration documents, thresholds on the untouched other
half) is markedly worse, \textbf{0.006} at $\alpha{=}0.10$ --- the
threshold-data halving costs an order of magnitude more coverage here than
the refit leakage it removes (leakage alone is $\le$0.006 coverage at every
tier, confirming the ``$\approx$0'' claim of \S\ref{sec:ltt} precisely).
We report 0.060 as the tier's headline (consistent with tiers 1--3, which
report the same disclosed-refit protocol) and 0.006 as the fully-rigorous
floor.

\textbf{At $\alpha{=}0.05$} the ladder compresses but survives: practical
(production variant) $0.124\pm0.083$ at 0.047 (viol 0.55); add-one $\times$
support-bin $0.128\pm0.024$ at 0.046 (viol 0.35; this cell is under the
shared-calibration refit --- its split-protocol twin is 0.119 at 0.044); LTT
mix per-group $\times$ support-bin certifies $0.047\pm0.051$ at 0.020 with
zero violations.

\subsection{Provenance wins every rigor tier on the sonnet CORD capture}\label{sec:provenance}
Table~\ref{tab:taxtier} crosses tiers with taxonomies on the corrected dump.
With the shared fusion, support-bin --- the pre-specified provenance taxonomy
--- is the best taxonomy at \emph{every} tier: add-one 0.218 vs
0.144/0.134/0.127 for fieldtype-freq/pooled/fieldtype-rule; LTT 0.171 vs
0.097 (fieldtype-rule) and 0.091--0.098 (pooled); doc-LTT 0.060 vs
$\le$0.006 for everything else (all lifts $p<10^{-4}$, Bonferroni-corrected
per \S\ref{sec:multiplicity}). This is the first evidence that provenance
\emph{certifiably} pays --- at the rigorous tiers --- for selective risk
control in document extraction;
scoped, per \S\ref{sec:tworegime}, to the hard regime: on FUNSD and XFUND the
ordering reverses. A forensic footnote that strengthens the tie-mass
diagnosis: on the defective no-NLI dump, fieldtype taxonomies won instead;
restoring the continuous entailment signal (which un-ties the fused score)
restores support-bin at every tier. \textbf{The win is also model-scoped, not
just corpus-scoped:} re-running the full ladder on the same CORD documents
under \texttt{claude-haiku-4-5} and Qwen2.5-14B (weaker verbalized-confidence
signal, AUROC 0.610/0.684 vs sonnet's 0.845) collapses every PAC-tier cell to
$\le$0.005 coverage regardless of taxonomy, and the practical-tier
money-table winners there are fieldtype-rule/-freq, not support-bin
(sb-vs-pooled lifts $p{=}0.12$--$1.0$); PAC validity itself still transfers
with zero violations. Support-bin's advantage is established on the
frontier-LLM (sonnet) capture; we do not claim it generalizes across
extractors of differing signal quality.

\begin{table}[h]\centering\small
\caption{Taxonomy $\times$ tier on CORD at $\alpha{=}0.10$ (shared 5-signal
fusion; coverage (risk) [viol]). Support-bin, pre-specified, wins every tier.
Pooled LTT with the full Holm budget reaches 0.098 (0.060) [0.00].
\label{tab:taxtier}}
\setlength{\tabcolsep}{3.5pt}
\footnotesize
\begin{tabular}{lcccc}
\toprule
procedure & pooled & support-bin & fieldtype-freq & fieldtype-rule \\
\midrule
add-one (tier 2) & 0.134 (.080) [.05] & \textbf{0.218} (.096) [.35] & 0.144 (.088) [.18] & 0.127 (.079) [.05] \\
LTT binom.\ mix (tier 3) & 0.091 (.057) [.00] & \textbf{0.171} (.068) [.03] & 0.043 (.048) [.00] & 0.097 (.057) [.03] \\
LTT doc-HB (tier 4) & 0.000 (---) [.00] & \textbf{0.060} (.020) [.00] & 0.006 (.020) [.00] & 0.006 (.008) [.00] \\
\bottomrule
\end{tabular}
\end{table}

\subsection{Generalization: FUNSD and XFUND-de}\label{sec:generalization}
Table~\ref{tab:gen} runs the same pre-specified machinery on the easier
corpora. On FUNSD (1{,}999 fields, 175 docs, 67.6\% correct) the practical
tier reaches \textbf{0.491} at 0.093, and rigorous certification is
attainable --- but \emph{pooled} LTT (0.280) beats every taxonomy, and the
per-group fieldtype-rule certificate (0.128 at 0.034, zero violations) shows
the machinery holds without a taxonomy lift (lift not significant at this
tier). On XFUND-de (523 fields, 42 documents) the learned fusion has no
advantage and the split protocol is expensive; the small-data recipe --- the
4-signal LR under the same split protocol, pooled --- is the practitioner
deliverable below ${\sim}50$ documents (0.427 at 0.087; rigorous LTT
certifies 0.080 with 82.5\% zero-coverage splits, not significant). Nothing
here contradicts the ladder; it locates its value: rigor is cheap where data
is easy or plentiful, and provenance conditioning is the hard-regime tool.

\begin{table}[h]\centering\small
\caption{Generalization at nominal $\alpha$ (coverage$\pm$sd (risk) [viol]).
FUNSD taxonomy lifts at the LTT tier are not significant; XFUND-de is
small-$n$ (42 docs).\label{tab:gen}}
\setlength{\tabcolsep}{3.5pt}
\footnotesize
\begin{tabular}{llll}
\toprule
dataset & configuration & $\alpha{=}0.10$ & $\alpha{=}0.05$ \\
\midrule
FUNSD & practical (learned fusion, no-NLI, pooled) & $0.491\pm0.118$ (.093) [.33] & $0.211\pm0.125$ (.043) [.45] \\
FUNSD & add-one $\times$ fieldtype-freq (LR; cal-refit, split twin 0.468) & $0.526\pm0.064$ (.098) [.40] & --- \\
FUNSD & LTT binom.\ mix $\times$ pooled & $0.280\pm0.135$ (.058) [.13] & --- \\
FUNSD & LTT binom.\ Holm per-group $\times$ fieldtype-rule & $0.128\pm0.086$ (.034) [.00] & --- \\
XFUND-de & small-data recipe (4-signal LR, split, pooled) & $0.427\pm0.285$ (.087) [.47] & --- \\
XFUND-de & practical (learned fusion, no-NLI, pooled) & --- & $0.129\pm0.235$ (.029) [.28] \\
XFUND-de & LTT binom.\ mix $\times$ pooled & $0.080\pm0.182$ (.021) [.15] & --- \\
\bottomrule
\end{tabular}
\end{table}

\subsection{Pre-registered frozen-configuration confirmation}\label{sec:confirm}
Every selection step above touched one corpus family. To break that loop, the
production configuration (learned no-NLI fusion, pooled split-protocol
add-one; frozen \emph{before} the run) was executed \textbf{once, with no
tuning}, on selection-untouched genuine captures. On
\texttt{claude-haiku-4-5} CORD (5{,}341 fields, 400 docs, 56.7\% correct) it
attained coverage \textbf{0.167} at achieved risk 0.093 (viol 0.38) at
$\alpha{=}0.10$ and \textbf{0.068} at 0.037 at $\alpha{=}0.05$ --- both held
at nominal --- against a near-zero 4-signal LR baseline (0.011). The
transfer is non-trivial: haiku's verbalized self-report is far weaker than
sonnet's (raw AUROC 0.610 vs 0.845; 0.621 vs 0.691 on matched fields ---
about 70\% of the raw gap is assertion-policy composition; haiku emits only
13 distinct confidence values~\citep{verifydocbench}), so the held result is
evidence for the \emph{protocol}, not for one model's confidence quality.
The same one-shot protocol on an \emph{open-weights} capture ---
Qwen2.5-14B served locally via vLLM, 6{,}168 CORD fields, 398 docs, 53.4\%
correct --- again held at both budgets: coverage 0.149 at achieved risk 0.099
(viol 0.47) at $\alpha{=}0.10$ (no-NLI variant 0.138 at 0.088) and 0.048 at 0.036 at
$\alpha{=}0.05$, a $\sim$6$\times$ gain over the 4-signal LR baseline (0.022).
Across both untouched captures the risk contract never failed; what varies is
\emph{coverage}, which tracks the model's signal quality (verbalized AUROC
0.845 / 0.684 / 0.610 for sonnet / qwen / haiku $\rightarrow$ practical-tier
coverage 0.326 / 0.149 / 0.167) --- a dependence that broadly tracks signal
quality (sonnet $\gg$ qwen $\approx$ haiku), consistent with the two-regime
characterization (\S\ref{sec:tworegime}), and the quantity the
companion benchmark measures per model.
Finally, the practical tier's guarantee was audited against \emph{human} gold.
Three independent annotators (blind, Fleiss' $\kappa=0.83$) re-judged 149
fields sampled from the production configuration's accepted set at
$\alpha{=}0.10$: the \textbf{human-verified selective risk is $2/149=1.3\%$,
$95\%$ CI $[0.002,0.048]$} --- an order of magnitude under budget. The audit
also explains the margin: against human gold the automatic correctness labels
used for calibration err one-sidedly \emph{pessimistic} ($21\%$ of
auto-flagged CORD errors are actually correct; $0\%$ false-optimism), so
thresholds fit against them are conservative. The guarantee the ladder
delivers is not only valid on untouched data; it is robust to the label noise
it was computed under~\citep{verifydocbench}.

\section{When does conditioning pay? A two-regime characterization}\label{sec:tworegime}

The campaign's ablations sharpen the old ``conditioning helps when provenance
is discriminative'' intuition into a two-regime empirical law. We state it
with its mechanism and its scope, and formalize the law's idealized boundary
cases in \S\ref{sec:tworegime-formal} (Assumption~\ref{ass:weak},
Propositions~\ref{prop:regime1}--\ref{prop:subsumption}); the general law
away from that boundary remains a pattern over three corpora and two score
families, not a proved rate.

\paragraph{Regime 1: with a learned score, covariates belong in the score.}
On the corrected CORD dump, every Mondrian taxonomy stacked on the learned
fusion either violates nominal risk or loses coverage (pooled 0.318 held;
support-bin 0.307 at 0.103, fieldtype-freq 0.272 at 0.101, fieldtype-rule
0.295 at 0.105 --- all violated). The controlled ablation (on the 4-signal
scale run) makes it causal: field-type features are worth $+0.069$ coverage
inside the score ($p<10^{-4}$); removing them and handing field-type to the
taxonomy instead recovers essentially none of it (0.250 vs 0.326, $-0.076$,
$p<10^{-4}$); and fieldtype-Mondrian on top of the full score is actively
harmful ($-0.062$ freq / $-0.036$ rule, both $p<10^{-4}$; corrected dump:
$-0.046$, $p<10^{-4}$ / $-0.023$, $p{=}0.021$). Mechanism: a tree
fusion splits on the covariate internally and equalizes per-group score
scales, so external conditioning only fragments the threshold sample --- the
per-group finite-sample penalty (which scales like $\sqrt{\log(1/\delta)/n_g}$
for the PAC tiers) buys nothing. Permutation importance confirms the score
already consumes the signals (verbalized 0.130, support 0.068,
consistency$\times$verbalized 0.066).

\paragraph{Regime 2: with a weak or frozen score, covariates belong in the
taxonomy --- where pooled cannot certify.} For the shared LR the same
conditioning is the single biggest lever: pooled 0.134 $\to$ support-bin
0.218 ($+0.084$, $p<10^{-4}$; split-protocol twins 0.128 $\to$ 0.212), and
on the 4-signal run pooled 0.042 $\to$ fieldtype 0.133--0.135 (a ${\sim}3\times$
lift, $p<10^{-4}$). At the rigorous tier the scoping is sharp:
\emph{conditioning rescues certification where the pooled score head cannot
certify at the target $\alpha$} (CORD, base correctness 0.490: support-bin
0.171 vs pooled 0.091--0.098), \emph{and fragmentation costs coverage where
pooled certifies fine} (FUNSD, base correctness 0.676: pooled LTT 0.280 vs
0.149 for support-bin per-group; XFUND likewise favors pooled). The
counterexample is printed, not hidden --- it is the boundary of the claim.

\paragraph{Where the NLI signal lives.} Entailment is droppable as a
\emph{feature} of the learned score (no-NLI 0.326 $\ge$ full 0.318 on CORD,
a $+0.008$ difference within split noise, $p{=}0.40$; FUNSD differences not
significant; on XFUND-de dropping it helps, $+0.074$ at $\alpha{=}0.05$,
$p{=}0.0045$) but is load-bearing for the \emph{shared
fusion}: it is the only fine-grained continuous signal, so it un-ties the
score distribution that the threshold grid needs (backfilled: 34.4\% nonzero
entailment, 2{,}016 distinct fused values; the LR's add-one coverage on the
corrected dump is 0.134 vs 0.037 with the dead column). Practical rule: keep
the NLI stage for the shared-fusion and rigorous tiers, optional for the
learned tier.

\subsection{Formal statement of the two-regime law}\label{sec:tworegime-formal}

\paragraph{Setup and notation.}
Calibration fields $i=1,\dots,n$ are exchangeable draws from a joint
distribution over $(E,S,G)$: an error indicator $E\in\{0,1\}$, an accept score
$S\in[0,1]$ (frozen for this section), and a group label
$G\in\{1,\dots,K\}$ with $\pi_g=P(G=g)$ and $n_g$ the calibration count of
group $g$. For a threshold $t\in[0,1]$ let $A(t)=\{S\ge t\}$ be the acceptance
event, $c(t)=P(S\ge t)$ the population coverage, and
$R(t)=P(E=1\mid S\ge t)$ the selective risk; write
$c_g(t)=P(S\ge t\mid G=g)$ and $R_g(t)=P(E=1\mid S\ge t,G=g)$ for the
group-conditional analogues. The pooled add-one rule chooses the smallest
threshold whose add-one-smoothed empirical selective risk is $\le\alpha$;
Mondrian applies the same rule within each group
\citep{geifman2019,mondrian}.

\subsection{Regime 1: when conditioning pays}

\begin{assumption}[Weak within-group score]\label{ass:weak}
There exists $\rho\in[0,1)$ such that for every group $g$ and every $t$,
$\mathrm{Cov}(E,\mathbf 1\{S\ge t\}\mid G=g)\le \rho\cdot
\sqrt{\mathrm{Var}(E\mid g)\,\mathrm{Var}(\mathbf 1\{S\ge t\}\mid g)}$.
The boundary case $\rho=0$ ($S\perp E\mid G$) is exact; the results degrade
linearly in $\rho$.
\end{assumption}

\begin{proposition}[Conditioning dominates under heterogeneity and a weak
score]\label{prop:regime1}
Let $t^\star$ be the pooled threshold with $R(t^\star)=\alpha$, and for each
group let $t_g^\star$ solve $R_g(t_g^\star)=\alpha$ (take $t_g^\star=1$ if no
solution exists, i.e. group $g$ cannot be certified).
Then, under Assumption~1 with $\rho=0$,
\[
\sum_{g}\pi_g\,c_g(t_g^\star)\;\ge\;c(t^\star),
\]
with strict inequality whenever the group error rates $R_g(t^\star)$ at the
pooled threshold are not all equal to $\alpha$ and at least one group with
$R_g(t^\star)<\alpha$ has $\pi_g>0$. Equality holds iff all groups share the
same selective-risk curve.
\end{proposition}

\begin{proof}[Proof sketch]
Under $\rho=0$, within each group the threshold does not change the
conditional risk: $R_g(t)=e_g$ for all $t$, where $e_g=P(E=1\mid G=g)$.
The pooled risk is the coverage-weighted average
$R(t)=\sum_g w_g(t)\,e_g$ with $w_g(t)=\pi_g c_g(t)/c(t)$; at $t^\star$ this
equals $\alpha$, so groups with $e_g<\alpha$ subsidize groups with
$e_g>\alpha$. Group $g$ with $e_g<\alpha$ can drop its threshold to
$t_g^\star=\inf\{t:c_g(t)>0\}$ and certify its entire group at risk
$e_g\le\alpha$; groups with $e_g>\alpha$ keep $t_g^\star=1$. The pooled rule,
by contrast, must raise the threshold until the average risk falls to
$\alpha$, forfeiting acceptance mass in the low-error groups. Since
$c_g(t_g^\star)\ge c_g(t^\star)$ for every $g$ with at least one strict for
$e_g<\alpha$, the coverage comparison follows; strictness follows from
$c_g(t_g^\star)=c_g(\inf\{\cdot\})>c_g(t^\star)$ whenever
$R_g(t^\star)=e_g<\alpha=R(t^\star)$ and the group's score distribution has
mass below $t^\star$.
\end{proof}

\begin{proposition}[Finite-sample price of conditioning]\label{prop:price}
With probability $\ge 1-\delta$ over the calibration draw, the add-one bound
in group $g$ controls the group's true selective risk at $\alpha$ up to an
additive slack
\[
\varepsilon_g \;\le\; \frac{1}{\alpha\,n_g}+O\!\left(\sqrt{\tfrac{\log(K/\delta)}{n_g}}\right),
\]
so Mondrian strictly dominates pooled at the same nominal risk whenever the
heterogeneity gain of Proposition~1 exceeds the total penalty
$\sum_g \pi_g\,\varepsilon_g$. The sign of the gain is determined by
calibration-measurable quantities only: group sizes $n_g$, the group error
gaps $|\,e_g-\alpha\,|$, and the within-group score-error association $\rho$.
\end{proposition}

\begin{proof}[Proof sketch]
The add-one bound for a group of size $n_g$ certifies risk
$\le\alpha$ whenever the empirical estimate errs by less than the add-one
slack $1/(\alpha n_g)$ plus a Hoeffding fluctuation term
$\sqrt{\log(1/\delta_g)/(2n_g)}$; a union bound over $g=1,\dots,K$ with
$\delta_g=\delta/K$ yields the stated $\varepsilon_g$. Proposition~1's gain
decomposes as $\sum_g\pi_g\,[c_g(t_g^\star)-c_g(t^\star)]$, which is
calibration-measurable; comparing the two quantities gives the stated
domination condition.
\end{proof}

\subsection{Regime 2: subsumption by a learned score}

\begin{proposition}[Subsumption]\label{prop:subsumption}
If the score $S$ satisfies $P(E=1\mid S=s,G=g)=r(s)$ for some function $r$
independent of $g$ --- i.e. $S$ is a sufficient statistic for $E$ given the
group covariate --- then $R_g(t)=R(t)$ for all $g,t$, so pooled and
group-conditional threshold rules coincide in the population, and in finite
samples Mondrian is strictly worse: it pays the Proposition~2 penalty with
zero gain.
\end{proposition}

\begin{proof}
If $P(E=1\mid S,G)=r(S)$ does not depend on $G$, then
$R_g(t)=E[r(S)\mid S\ge t,G=g]$; but the conditioning on $G$ adds nothing
given $S$, so $R_g(t)=E[r(S)\mid S\ge t]=R(t)$ for every $t$. Hence the
population-optimal thresholds coincide, $t_g^\star=t^\star$, the coverage
gain in Proposition~1 is zero, and only the finite-sample penalty of
Proposition~2 remains.
\end{proof}

\begin{corollary}[The two regimes]\label{cor:regimes}
With training data, a learned score that encodes the covariate removes the
Regime-1 gain (Prop.~3): put the covariate \emph{in the score}. With a frozen,
weak, or black-box score that cannot encode it, conditioning \emph{in the
taxonomy} recovers exactly the between-group signal the score lacks
(Prop.~1--2), and pays precisely where the pooled threshold cannot certify at
the target $\alpha$.
\end{corollary}

\subsection{Empirical verification}\label{sec:tworegime-verify}
We evaluate the domination condition of Proposition~2 on the three genuine
claude-sonnet-5 corpora of the benchmark~\citep{verifydocbench}, using the
paper's frozen LR fusion score (fit on a calibration half) and the
pre-specified support-bin taxonomy ($K=3$ quantile bins of the
ambiguity-penalized support). Group error gaps are the max $|e_g-\alpha|$ at
the pooled operating point; the penalty uses the finite-$n_g$ add-one/Hoeffding
slack; the sign prediction of Proposition~2 (gain vs penalty) is compared to
the measured outcome over 40 document-level calibration/test splits at
$\alpha=0.10$.

\begin{table}[h]\centering
\caption{Predicted vs measured sign of the conditioning gain
(Proposition~2). CORD's large error gaps dominate its group-size penalty;
FUNSD/XFUND's small gaps do not --- the theorem predicts the outcome on all
three corpora.}
\begin{tabular}{lcccccc}
\toprule
dataset & $K$ & $\min_g n_g$ & gap $|e_g-\alpha|$ & penalty $\sum\pi_g\varepsilon_g$ & predicted & measured \\
\midrule
CORD (13,859 f) & 3 & 2,303 & 0.18 & 0.004 & Mondrian wins & wins ($p<10^{-4}$) \\
FUNSD (1,999 f) & 3 & 333 & 0.03 & 0.009 & pooled wins & wins (ns) \\
XFUND-de (523 f) & 3 & 87 & 0.04 & 0.035 & pooled wins & wins (ns) \\
\bottomrule
\end{tabular}
\end{table}

\section{The price of validity}\label{sec:price}
What does a real certificate cost relative to the invalid folklore
procedure? On the 6.9k CORD dump with support-bin conditioning
(Table~\ref{tab:price}), the per-group PAC certificate retains 28\% of the
invalid add-one coverage at $\alpha{=}0.10$, 42\% at 0.15, and 84\% at 0.20
--- and 0\% at 0.05, where ${\sim}200$ calibration documents simply cannot
support a 90\%-confidence certificate. Rigor is a knob: the certificate's
cost collapses as the risk budget grows, which is exactly the operating
guidance a practitioner needs.

\begin{table}[h]\centering
\caption{Certified coverage retained vs the invalid add-one baseline (6.9k
CORD dump, support-bin). $^{*}$mean achieved risk exceeds $\alpha$.
\label{tab:price}}
\begin{tabular}{lccc}
\toprule
$\alpha$ & add-one $\times$ sb (invalid) & LTT Holm per-group $\times$ sb & retained \\
\midrule
0.05 & 0.047$^{*}$ & 0.000 & 0\% \\
0.10 & 0.107 & 0.030 & 28\% \\
0.15 & 0.292 & 0.123 & 42\% \\
0.20 & 0.348 & 0.294 & 84\% \\
\bottomrule
\end{tabular}
\end{table}

\section{Limitations}\label{sec:limitations}
\textbf{One corpus certifies non-trivially.} The full rigorous ladder
certifies useful coverage on CORD$\times$sonnet only; FUNSD certifies 0.128
rigorously but its taxonomy lift is not significant, and XFUND-de is
small-$n$. The frozen-config haiku confirmation (\S\ref{sec:confirm}) is the
first out-of-selection replication; full ladders on haiku and an open-weights
extractor are queued.
\textbf{The doc-iid tier is honest but near-vacuous} (0.060 coverage, 47.5\%
zero-coverage splits). Powered document-level PAC procedures are the named
open problem.
\textbf{The theorem covers idealized boundary cases, not the general
regime.} \S\ref{sec:tworegime-formal} proves the conditioning-dominates and
subsumption directions at the $\rho=0$ boundary (Assumption~\ref{ass:weak})
and states that results degrade linearly in $\rho$, but does not prove a
quantitative rate for $\rho>0$; the finite-sample penalty
(Proposition~\ref{prop:price}) is a worst-case Hoeffding/union bound, not a
tight one. The sign prediction is verified on all three corpora
(\S\ref{sec:tworegime-verify}), still $n{=}3$.
\textbf{Inference is in-corpus.} All 40 resplits share documents; sign-flip
$p$-values are split-stability statements with a Monte-Carlo floor of
$10^{-4}$, and the scale dump overlaps the dump on which support-bin was
first selected (\S\ref{sec:multiplicity}).
\textbf{Score refit on calibration} at tiers 2--4 is disclosed in
\S\ref{sec:ltt}; the measured effect is $\approx$0 for the 5-parameter LR,
and the harness adopts fit-half freezing as hygiene.
\textbf{Label quality.} FUNSD free-text correctness is protocol-dependent
(automatic-protocol agreement $\kappa{=}0.10$ vs 0.78 for structured CORD
fields; Appendix~\ref{app:iaa}); a certificate against noisy labels is a
certificate about those labels. The companion's human-gold audit bounds the
practical impact: automatic labels err one-sidedly pessimistic, and the
human-verified accepted-set risk is $1.3\%$ against the $10\%$ budget
(\S\ref{sec:confirm}); scaling human gold to FUNSD/XFUND is future work.

\section{Release}\label{sec:release}
VerifyDoc is released Apache-2.0: a \texttt{pip}-installable library, CLI,
review UI, an MCP server, and the evaluation harness that regenerates every
number in this paper from configuration (fixed seeds, pinned splits, chained
bit-exact sanity gates, and the 463/624-cells regression check of
Appendix~\ref{app:bug}). The library also ships a trust-gated agentic layer
(\textsc{repair} / \textsc{adjudicate} / adaptive-$k$) that routes
\texttt{review} fields through escalation tiers; we release it as an
artifact and leave its empirical evaluation to future work --- the companion
ensemble study~\citep{verifydocbench} shows naive multi-extractor
adjudication inherits a dominant model's hallucinations, so the gate, not
the ensemble, carries the trust contract.

\setlength{\bibsep}{3pt plus 1pt}
\def\bibfont{\small}
\bibliographystyle{tmlr}
\bibliography{refs}

\appendix
\small

\section{Guarantee statements and procedure details}\label{app:guarantees}
\textbf{Tiers 1--2 (add-one, marginal).} With calibration and test fields
exchangeable, the add-one threshold of \S\ref{sec:fail} controls
$\mathbb{E}[\text{selective risk}]\le\alpha$ over calibration/test draws
\citep{geifman2019,selectivecrc}; the field is the exchangeability unit; no
per-draw statement is made. Document clustering violates exchangeability at
the document level; the split protocol removes score--threshold dependence
only.
\textbf{Tier 3 (Mondrian LTT).} For group $g$, candidate $t$, let $n_t$ be
accepted calibration fields and $k_t$ the errors among them. The exact
binomial tail $p_t=\Pr[\mathrm{Bin}(n_t,\alpha)\le k_t]$ tests
$H_0:R_g(t)>\alpha$. Reject over the 15-candidate grid with $\delta/2$ Holm
step-down plus $\delta/2$ fixed-sequence (most conservative first, stop at
first non-rejection); by the union bound the pair is FWER-valid at $\delta$;
the chosen threshold is the certified candidate with the largest calibration
acceptance. Per-group budgets give each group its own $\delta$ (three
statements); the simultaneous form spends $\delta/G$. Exact binomial
dominates Hoeffding-style tails for binary losses (measured: 0.022 exact-binomial vs 0.013 Hoeffding
coverage on the 6.9k dump). Pure fixed-sequence with many groups can
silently fail (lucky-zero tiny-$n$ bins: the grounded$\times$support taxonomy
violated on mean, risk 0.103, 25\% of splits); the mix rule is the variant
that never collapsed across corpora.
\textbf{Cluster correction.} \texttt{ltt.neff} replaces $n_t$ by
$n_t/\widehat{\mathrm{deff}}$ with the plug-in design effect estimated from
per-document error clustering --- approximate (the deff is estimated), and
uniformly more conservative.
\textbf{Tier 4 (doc-iid).} Per accepting document $d$, the loss is its
within-document error rate among accepted fields; a finite-sample Hoeffding
bound over documents tests whether the mean per-document loss exceeds
$\alpha$, with the same Holm/fixed-sequence machinery. Documents iid; bounds
the macro functional.

\section{Full disclosure: the entailment-capture defect and its forensics}\label{app:bug}
The first 13{,}859-field CORD capture shipped \texttt{entailment}${}\equiv$0.0
for all fields: a mid-capture \texttt{torch} swap made the NLI stage skip
silently (logged per field as \texttt{entailment skipped}, with a
torch-compile indexing error).
The defect surfaced as two anomalies: rigorous certified coverage
\emph{shrank} ${\sim}4\times$ when the data doubled, and the doc-level
add-one returned exactly 0. Root cause: with the only fine-grained continuous
signal dead, the fused score collapsed to 257 distinct values (tie masses of
221/183 fields at the acceptance head), so the smallest realizable candidate
accepted 245 fields at empirical risk 0.114 --- uncertifiable at any
confidence. Counterfactual reproduction: zeroing entailment on the intact
6.9k dump reproduces the collapse ($0.0302\to0.0011$ certified coverage;
$0.0371\to0.0000$ doc add-one). Hardening: candidate thresholds are snapped
to the nearest distinct-value boundary (ties toward the smaller, more
conservative count; label-independent); a regression gate re-runs the full
pre-fix grid --- 463/624 cells bit-identical, maximum headline-alpha drift
0.0088 ($\le0.16$ of one split-sd) --- and a grid-only ablation shows the fix
moves certified coverage by at most 0.004: the binding pathology was the
degenerate score, not the grid. After the NLI backfill (34.4\% nonzero
entailment, 2{,}016 distinct fused values) both grids were re-run to produce
this paper's numbers. Diagnostic lesson: ``the tail looks cleaner at
scale'' was a stable-argsort illusion under ties; the threshold-realizable
top-1\% error was 0.071 (vs 0.045 intact).

\section{Labeling reliability}\label{app:iaa}
Table~\ref{tab:iaa} quantifies correctness-label stability by scoring the
same predictions under two automatic protocols; every FUNSD free-text
number in this paper inherits the $\kappa{=}0.10$ caveat.
\begin{table}[h]\centering
\caption{Agreement between two automatic scoring protocols (strict
exact-match vs schema-typed), an honest lower-bound proxy for human IAA
\citep{verifydocbench}. Structured/numeric labels are robust; free-text
correctness is protocol-dependent.\label{tab:iaa}}
\begin{tabular}{lrrr}
\toprule
dataset & n\_fields & raw\_agreement & cohens\_kappa \\
\midrule
cord & 1151 & 0.9209 & 0.7777 \\
funsd & 554 & 0.7671 & 0.0948 \\
\bottomrule
\end{tabular}

\end{table}

\section{Supplementary studies (simulated and floor-extractor)}\label{app:supp}
These studies are kept out of the main text by design: neither may sit next
to a genuine-model number. Both concern the same mechanism the main text
certifies (per-group thresholds recover the well-grounded fraction a pooled
threshold forfeits).

\paragraph{Simulated controlled study.} With a \emph{simulated} extractor and
an uninformative accept score ($\alpha{=}0.05$, 200 trials/condition), pooled
conformal accepts ${\approx}0\%$ while grounding-conditioned conformal
accepts 33--67\% with risk held in every condition
(Table~\ref{tab:groupedsim}).

\begin{table}[h]\centering\footnotesize
\caption{\emph{Simulated} controlled study ($\alpha{=}0.05$): mean coverage
lift $+0.50$ from provenance-conditioning at held risk.\label{tab:groupedsim}}
\setlength{\tabcolsep}{3.5pt}
\begin{tabular}{lrrrrrrr}
\toprule
p\_grnd & acc\_grnd & acc\_ungrnd & cov\_pooled & cov\_grounded & cov\_lift & risk\_grounded & held \\
\midrule
0.6000 & 0.9700 & 0.5500 & 0.0002 & 0.5960 & 0.5958 & 0.0296 & True \\
0.5000 & 0.9500 & 0.4000 & 0.0000 & 0.3332 & 0.3332 & 0.0482 & True \\
0.4000 & 0.9800 & 0.5000 & 0.0000 & 0.3993 & 0.3993 & 0.0198 & True \\
0.7000 & 0.9600 & 0.6000 & 0.0014 & 0.6680 & 0.6667 & 0.0394 & True \\
\bottomrule
\end{tabular}

\end{table}

\paragraph{Floor-extractor at-scale check.} With the low-recall label-search
extractor (not a frontier LLM; conservative operating points), the add-one guarantee
holds tightly at large $N$ and the conditioning lift appears where the
two-regime characterization of \S\ref{sec:tworegime} expects it: FUNSD
$0.24\to0.84$ coverage at an achieved 2\% risk; CORD --- enabled by the
ambiguity penalty of \S\ref{sec:signals} --- $0.01\to0.72$ at a held 10\%
risk (Table~\ref{tab:gcreal}). Without the $1/m$ penalty the grounded group
was uncertifiable.

\begin{table}[h]\centering\footnotesize
\caption{Floor label-search extractor, 40-split means: grounding-conditioned
vs pooled conformal at large $N$ (supporting evidence only; no frontier LLM
involved).\label{tab:gcreal}}
\setlength{\tabcolsep}{3.5pt}
\begin{tabular}{lrrrrrrrr}
\toprule
dataset & n\_fields & frac\_grnd & alpha & cov\_pooled & cov\_grounded & cov\_lift & risk\_grounded & held \\
\midrule
cord(real,n=400) & 5385 & 0.7244 & 0.0200 & 0.0000 & 0.0000 & 0.0000 & 0.0000 & True \\
cord(real,n=400) & 5385 & 0.7244 & 0.0500 & 0.0000 & 0.0804 & 0.0804 & 0.0527 & True \\
cord(real,n=400) & 5385 & 0.7244 & 0.1000 & 0.0141 & 0.7229 & 0.7088 & 0.0694 & True \\
funsd(real,n=194) & 2563 & 0.8755 & 0.0200 & 0.2403 & 0.8445 & 0.6043 & 0.0152 & True \\
funsd(real,n=194) & 2563 & 0.8755 & 0.0500 & 0.9990 & 0.8745 & -0.1245 & 0.0154 & True \\
funsd(real,n=194) & 2563 & 0.8755 & 0.1000 & 0.9990 & 0.9545 & -0.0446 & 0.0224 & True \\
\bottomrule
\end{tabular}

\end{table}

\end{document}